\documentclass[journal]{ieeetran}

\usepackage{cite}
\usepackage{amsmath,amssymb,amsfonts,amsthm}
\usepackage{mathtools,commath}
\usepackage{algorithmic}
\usepackage{graphicx}
\usepackage{algorithm,algorithmic}
\usepackage{textcomp}
\usepackage[mathcal]{euscript}
\usepackage{comment}

    \usepackage[usenames]{color}
    \definecolor{plum}  {rgb}{.4,0,.4}
    \definecolor{BrickRed} {rgb}{0.6,0,0}
	\definecolor{DarkBlue} {rgb}{0,0,0.6}
\usepackage[plainpages=false,pdfpagelabels,colorlinks=true,linkcolor=BrickRed,citecolor=plum,urlcolor=DarkBlue]{hyperref}

\def\ddefloop#1{\ifx\ddefloop#1\else\ddef{#1}\expandafter\ddefloop\fi}

\def\ddef#1{\expandafter\def\csname b#1\endcsname{\ensuremath{\boldsymbol{#1}}}}
\ddefloop ABCDEFGHIJKLMNOPQRSTUVWXYZabcdefghijklmnopqrstuvwxyz\ddefloop

\def\ddef#1{\expandafter\def\csname c#1\endcsname{\ensuremath{\mathcal{#1}}}}
\ddefloop ABCDEFGHIJKLMNOPQRSTUVWXYZ\ddefloop

\def\ddef#1{\expandafter\def\csname s#1\endcsname{\ensuremath{\mathsf{#1}}}}
\ddefloop ABCDEFGHIJKLMNOPQRSTUVWXYZ\ddefloop

\def\Reals{{\mathbb R}}

\def\Tr{{\mathsf T}} 

\newtheorem{theorem}{Theorem}

\newtheorem{lemma}{Lemma}

\newtheorem{remark}{Remark}

\begin{document}

\title{Expressivity of Quadratic Neural ODEs}

\author{Joshua Hanson and Maxim Raginsky
\thanks{This work was supported in part by the NSF under award CCF-2106358 (``Analysis and Geometry of Neural Dynamical Systems'') and in part by the Illinois Institute for Data Science and Dynamical Systems (iDS${}^2$), an NSF HDR TRIPODS institute, under award CCF-1934986.}
\thanks{Joshua Hanson was with University of Illinois, Urbana, IL 61801 USA. He is now with Error Corp and Sandia National Laboratories (e-mails: josh@error-corp.com and jmhanso@sandia.gov). }
\thanks{Maxim Raginsky is with the Department of Electrical and Computer Engineering and the Coordinated Science Laboratory, University of Illinois, Urbana, IL 61801, USA (e-mail: maxim@illinois.edu).}}

\maketitle

\begin{abstract}

This work focuses on deriving quantitative approximation error bounds for neural ordinary differential equations having at most quadratic nonlinearities in the dynamics. The simple dynamics of this model form demonstrates how expressivity can be derived primarily from iteratively composing many basic elementary operations, versus from the complexity of those elementary operations themselves. Like the analog differential analyzer and universal polynomial DAEs, the expressivity is derived instead primarily from the ``depth'' of the model. These results contribute to our understanding of what depth specifically imparts to the capabilities of deep learning architectures.

\end{abstract}


	\thispagestyle{empty}

\section{Introduction}\label{sec:introduction}

Historical works studying ``differential analyzers'' --- mechanical analog computers engineered to solve differential equations using disc-and-wheel integrators --- demonstrate that even extremely simple elementary operations are sufficient to efficiently synthesize rich function classes when these operations are interconnected and iteratively composed \cite{Bush_1927}, \cite{Bush_1931}, \cite{Shannon_1941}. These early computers achieved multiplication as a special case of integrating constant functions, and derived addition and subtraction from averagers implemented via differential drive shafts. Later analysis characterized the classes of functions generable by analog computers \cite{Pour-El_1974}, which are notably very large. Analog-generable function classes are connected to the solution spaces of so-called universal differential equations \cite{Rubel_1981}, \cite{Duffin_1981} --- polynomial differential algebraic equations whose set of admissible solutions, which is finitely parametrized by the set of initial conditions, can approximate arbitrary continuous functions. Universal differential equations can also be represented as equivalent quadratic vector fields through lifting transformations.

Similar in spirit to universal differential equations, there exist neural networks with bounded width and bounded depth that achieve universal approximation using only a finite number of parameters --- that is, the number of parameters required does not depend on the approximation error tolerance. This case was first studied in \cite{Maiorov_1999}, where the main theorem guarantees the existence of an analytic, strictly monotone, sigmoidal activation function such that the set of functions computable by bounded-width nets with two hidden layers using this activation are dense in the set of compactly supported, continuous functions. Here, all the complexity is pushed into the activation function, which does not satisfy any finite-order differential equation. This result is similar in spirit to the theorem in complex analysis that states that the Riemann zeta function can approximate arbitrary holomorphic functions with no zeros on certain compact sets of a strip of the complex plane by simply vertically shifting the domain \cite{Voronin_1975}. In this case, universal approximation is achieved using only a single real parameter, with the tradeoff being that the Riemann zeta function itself is extraordinarily rich and thus can be challenging to compute.

More practical activation functions for which finitely parametrized nets with bounded width, bounded depth, and fixed weights are universal approximators can also be constructed algorithmically \cite{Guliyev_2018}. Although this construction depends on the approximation tolerance, the procedure can be implemented using a finite number of elementary operations on a computer. In contrast to designing intricate activations in advance, elementary operations can instead be built directly into the network, resulting in a model with very simple layers that derives its expressivity primarily from its depth instead of its activations or its width.

The problem of function approximation can be viewed from a control-theoretic perspective with the help of continuous-time idealizations of deep neural nets, so-called neural ODEs (see \cite{Ruiz_Balet_2023} for a control-oriented overview). A neural ODE is simply a controlled dynamical system of the form
\begin{align*}
	\dot{y}(t) = f(y(t),\theta(t)),
\end{align*}
augmented with a linear read-in map $\phi$ that takes a $d$-dimensional input $x \in \Reals^d$ to a $W$-dimensional initial condition $y(0) = \phi(x)$ and a linear read-out map $\psi$ that takes the state $y(T)$ at some fixed terminal time $T > 0$ to a scalar output $\hat{f}(x) := \psi(y(T))$. The finite-dimensional controls $\theta(t)$, for $t \in [0,T]$, are then selected so that $\hat{f}(x)$ approximates a given target $f(x)$ in a suitable sense for all $x$ in some compact domain $K \subset \Reals^d$.

Selecting the controls to perform function approximation can be posed as an optimal control problem, where the terminal cost to minimize (or terminal endpoint constraint to drive to zero) is the deviation between an ensemble of target outputs and those outputs resulting from steering an ensemble of provided inputs (initial states) under the candidate control law \cite{Han_2018}, \cite{Agrachev_2022}. Universal approximation results that guarantee function approximation for deep networks can likewise be interpreted as statements establishing ensemble controllability of the underlying systems, and vice versa \cite{Sontag_1997}, \cite{Tabuada_2023}. These works make use of system-theoretic methods to guarantee the existence of an open-loop control to perform the desired function approximation or interpolation. Specifically, for a given approximation accuracy $\epsilon$, the idea in these works is to replace the domain $K$ by a finite grid of resolution $\epsilon$, then use Lie-theoretic controllability arguments to argue that there exists a single open-loop control to steer each point in the grid to its image under the target $f$, and then show that this control suffices to give an $\epsilon$-approximation to $f$ on all of $K$. By contrast, in this work we will instead use a direct construction and describe the controls explicitly.

In this paper, we study the role of depth in expressivity of deep networks by developing two main approximation error bounds for quadratic neural ODEs with simple nonlinearities but unlimited composition. In Section \ref{sec:model}, we describe the model architecture. The main results are stated and proved in Section \ref{sec:main_results}. The first theorem provides quantitative approximation error bounds for smooth functions in a unit Sobolev ball. The second theorem provides approximation error bounds for standard deep feedforward nets. These constructions are based on building certain elementary functions such as $\tanh$, $\exp$, $\ln$ etc. as solutions to quadratic ODEs, then using these pieces to generate, for example, a partition of unity supporting local polynomial approximations (in the case of approximating functions in a Sobolev ball), or constructing alternating compositions of affine maps and nonlinear activation functions (in the case of approximating deep feedforward nets). These results demonstrate that deep models do not require very complex individual layers or dynamics to achieve high overall expressivity, and makes a step towards understanding quantitatively what is contributed by ``depth'' to the approximation capability of a model in contrast to what is contributed by the ``width.''

\section{Model description}\label{sec:model}

Consider a time-varying quadratic neural ODE where the input $x \in \Reals^d$ enters through the initial condition as follows:
\begin{equation}\label{eq:neural_ode}
\begin{split}
    \dot{y}_i(t) &= \sum_{j=1}^W a_{ij}(t) y_j(t) + \sum_{j,k=1}^W q_{ijk}(t) y_j(t) y_k(t) + b_i(t) \\
    y_i(0) &= \begin{cases} x_i & \text{for } i=1,\dots,d \\ 0 & \text{for } i=d+1,\dots,W \end{cases}
\end{split}
\end{equation}
where the controls ${a_{ij} : [0,T] \to \Reals}$, ${q_{ijk} : [0,T] \to \Reals}$, ${b_i : [0,T] \to \Reals}$ for ${i,j,k = 1,\dots,W}$ serve the role of the weights and biases in a traditional feedforward net. The scalar output is given by a linear read-out of the terminal state ${c^\Tr y(T) \in \Reals}$ for some vector ${c \in \Reals^W}$. The state dimension $W$ is analogous to the ``width'', and the time horizon $T$ is analog to the ``depth.'' The time variable $t$ can be interpreted as a continuous layer index.

The purpose of considering this model architecture in particular is to demonstrate that high complexity nonlinear elements (activation functions) are not necessary to attain high expressivity. Instead, iterative composition (deep networks) of even very simple nonlinear operations is sufficient, and perhaps the most efficient method, to generate arbitrarily expressive features. We will show that piecewise-constant controls ${a_{ij} : [0,T] \to \Reals}$, ${q_{ijk} : [0,T] \to \Reals}$, ${b_i : [0,T] \to \Reals}$ can be explicitly constructed such that the resulting map sending initial states to terminal outputs $\hat{f} : \Reals^d \to \Reals$, $x \mapsto c^\Tr y(T)$ can approximate any continuous function $f : \Reals^d \to \Reals$ on a compact set. Naturally, the complexity of the controls depends on the required approximation error and the regularity of the desired function to approximate, which will be quantified in the main theorems.

The time horizon $T$ above can always be normalized to 1 by rescaling the dynamics, so the ``stiffness'' of the ODE --- which might be quantified in terms of maximum magnitude of the eigenvalues of the Jacobian of the dynamics, which is in turn directly related to the magnitude of the controls $a_{ij}, q_{ijk}, b_i$ and initial condition $x$ --- is perhaps a more appropriate analogy for the depth of the net, because this affects how many serialized computational operations must be performed to actually evaluate the neural ODE when using a numerical discretization. High-frequency, high-gain controls could be used to constructively approximate Lie bracket directions in control-affine systems; see \cite{Tabuada_2023}. Although the time-horizon may be small (based on the control gain), this leads to ``stiff'' ODEs that require a large number of timesteps to resolve via numerical integration, thus these models could be interpreted to have large ``depth.'' Simplifying the control functions yields shallower discretized nets because the ODEs become less stiff and ultimately require fewer ``time-steps'' to evaluate --- the constructions in our results attempt to respect this consideration.

\section{Main results}\label{sec:main_results}

In this section, we study the expressivity of time-varying quadratic neural ODEs by building quantitative estimates of the approximation error for some function classes of interest. The first theorem adapts the construction in \cite{Yarotsky_2017} for deep ReLU net approximations of functions in the unit ball of the Sobolev space $\cW^{n,\infty}([0, 1]^d)$ based on local truncated Taylor series stitched together using a partition of unity. The second theorem constructs an approximation for standard feedforward neural nets with $\tanh$ activations, which can be extended to networks with other activation functions that can be expressed as (or approximated by) the solution to a quadratic ODE.

In the following lemmas we will construct time-invariant quadratic ODEs to realize the nonlinear functions $\tanh$ and $\ln$ and multivariate monomials of arbitrary degree. Then, in proving the main results, we will apply these constructions to design explicit piecewise-constant controls for a quadratic neural ODE to approximate arbitrary smooth functions on $[0,1]^d$ with bounded Sobolev norm (Theorem \ref{thm:sobolev}), and deep feedforward nets with $\tanh$ activations (Theorem \ref{thm:feedforward}).

\begin{lemma}\label{lem:tanh}
	Let $a,b \in \Reals$. The initial value problem (IVP)
	\begin{alignat*}{2}
		\dot{y}_1 &= 0,\quad & y_1(0) &= \xi \\
		\dot{y}_2 &= (a y_1 + b) - y_2 y_3,\quad & y_2(0) &= 0 \\
		\dot{y}_3 &= (a y_1 + b)^2 - y_3^2,\quad & y_3(0) &= 0
	\end{alignat*}
	has time-1 solution
	\begin{align*}
		y_1(1) &= \xi \\
		y_2(1) &= \tanh(a \xi + b) \\
		y_3(1) &= (a \xi + b) \tanh(a \xi + b).
	\end{align*}
\end{lemma}

\begin{proof}
	Since $\dot{y}_1 = 0$, we can treat $a y_1 + b$ as a constant. We can construct $\tanh$ as the solution of the following scalar ODE
	\begin{equation*}
		\dot{y}_2 = (a y_1 + b)(1-y_2^2),
	\end{equation*}
	which satisfies
	\begin{align*}
		y_2(0) &= 0 \\
		y_2(1) &= \tanh(a \xi + b).
	\end{align*}
	To reduce the order of the cubic term $y_1 y_2^2$, define an additional state $y_3 = (a y_1 + b) y_2$, then
	\begin{equation*}
		\dot{y}_3 = (a y_1 + b) \dot{y_2} = (a y_1 + b)^2 - y_3^2.
	\end{equation*}
	The expression for $y_3(1)$ in the statement of the lemma can be verified via direct substitution or separation of variables.
\end{proof}

\begin{lemma}\label{lem:ln}
	Let $\xi > 0$. The IVP
	\begin{alignat*}{2}
		\dot{y}_1 &= y_2, \quad & y_1(0) &= 0 \\
		\dot{y}_2 &= -y_2^2,\quad & y_2(0) &= \xi - 1
	\end{alignat*}
	has time-1 solution
	\begin{align*}
		y_1(1) &= \ln(\xi) \\
		y_2(1) &= (\xi-1) / \xi.
	\end{align*}
\end{lemma}

\begin{proof}
	The scalar ODE $\dot{y}_2 = -y_2^2$ has general solution
	\begin{equation*}
		y_2(t) = \frac{y_2(0)}{y_2(0)t + 1},
	\end{equation*}
	which can be derived using separation of variables. Integrating this expression yields
	\begin{equation*}
		y_1(1) = y_1(0) + \int_0^1 \frac{y_2(0)}{y_2(0)t + 1} \dif t = \ln(y_2(0) + 1) = \ln(\xi).
	\end{equation*}
	The expression for $y_2(1)$ can be verified by substituting $y_2(0) = \xi - 1$ and $t = 1$ into the general solution.
\end{proof}

\begin{remark}
	In the proof of Lemma \ref{lem:ln}, for initial conditions $y_2(0) < 0$, the solution of the ODE $\dot{y}_2 = -y_2^2$ diverges in finite time; however the escape time occurs after time $1$ if $y_2(0) > -1$ (which holds provided that $\xi > 0$). While not strictly necessary, strengthening the assumption ${\xi > 0}$ to ${\xi > 1}$ eliminates this concern, as this causes the escape time to become negative. In the forthcoming proof of Theorem \ref{thm:sobolev}, this amounts to choosing $\Delta > 1$ and adding $1$ to each function $\psi_k$. Constructing the approximation $\hat{f}$ then will require forming additional monomials in $\psi_0+1,\dots,\psi_N+1$, but these can be generated using the same procedure as outlined in the proof.
\end{remark}

\begin{lemma}\label{lem:multiply}
	Let $w \in \Reals^d$, $b \in \Reals$, and $\xi \in (0,1]^d$. The IVP
	\begin{alignat*}{2}
		\dot{y}_{1:d} &= 0,\quad & y_{1:d}(0) &= \ln(\xi) \\
		\dot{y}_{d+1} &= (w^\Tr y_{1:d} + b) y_{d+1},\quad & y_{d+1}(0) &= 1
	\end{alignat*}
	has time-1 solution
	\begin{align*}
		y_{1:d}(1) &= \ln(\xi) \\
		y_{d+1}(1) &= e^b \xi_1^{w_1} \cdots \xi_d^{w_d}.
	\end{align*}
\end{lemma}

\begin{proof}
	We use the notation $y_{1:d}$ to represent the vector $(y_1,\dots,y_d) \in \Reals^d$. Since $\dot{y}_{1:d} = 0$, we can treat $w^\Tr y_{1:d} + b$ as a constant. Then we have
	\begin{align*}
		y_{m+1}(1) &= \exp( w^\Tr y_{1:d}(0) + b ) y_{m+1}(0) \\
		&= \exp( \ln(\xi_1^{w_1}) + \cdots + \ln(\xi_d^{w_d}) + b ) \\
		&= e^b \xi_1^{w_1} \cdots \xi_d^{w_d}.
	\end{align*}
	Notice that we can implement the inverse flow map --- which performs division instead of multiplication --- by making the substitutions $w \gets -w$ and $b \gets -b$.
\end{proof}

\begin{lemma}\label{lem:bootstrap}
	Let $c,\Delta > 0$ and $x \in [0,1]^d$. There is a piecewise-constant control with 7 duration-1 segments that steers \eqref{eq:neural_ode} from the initial state
	\begin{alignat*}{2}
		y_{j}(0) &= x_j \\
		y_{d+1:2(N+2)d}(0) &= 0
	\end{alignat*}
	to the terminal state
	\begin{align*}
		y_{1:(N+2)d}(7) &= 1 \\
		y_{(N+2)d + j}(7) &= \ln(x_j + \Delta) \\
		y_{(N+3)d + (k+1)j}(7) &= \ln(\psi_k(x_j)),
	\end{align*}
	for ${j=1,\dots,d}$ and ${k=0,\dots,N}$, where the functions ${\{\psi_k : [0, 1] \to (0,1) \}_{k \in \{0,\dots,N\}}}$ are defined by
	\begin{align*}
		\psi_0(x) &= \frac{1}{2} - \frac{1}{2} \tanh \left( c \left( x - \gamma_1 \right) \right) \\
		\psi_k(x) &= \frac{1}{2} \tanh \left( c \left( x - \gamma_k \right) \right) - \frac{1}{2} \tanh\left( c \left( x - \gamma_{k+1} \right) \right) \\
		\psi_N(x) &= \frac{1}{2} + \frac{1}{2} \tanh \left( c \left( x - \gamma_N \right) \right)
	\end{align*}
	with $\gamma_k = \frac{2k-1}{2N}$, so that they form a partition of unity, i.e., $\sum_{k=0}^N \psi_k(x) = 1$ for every $x \in [0,1]$.
\end{lemma}

\begin{figure}
	\centering
	\includegraphics[width=0.5\linewidth]{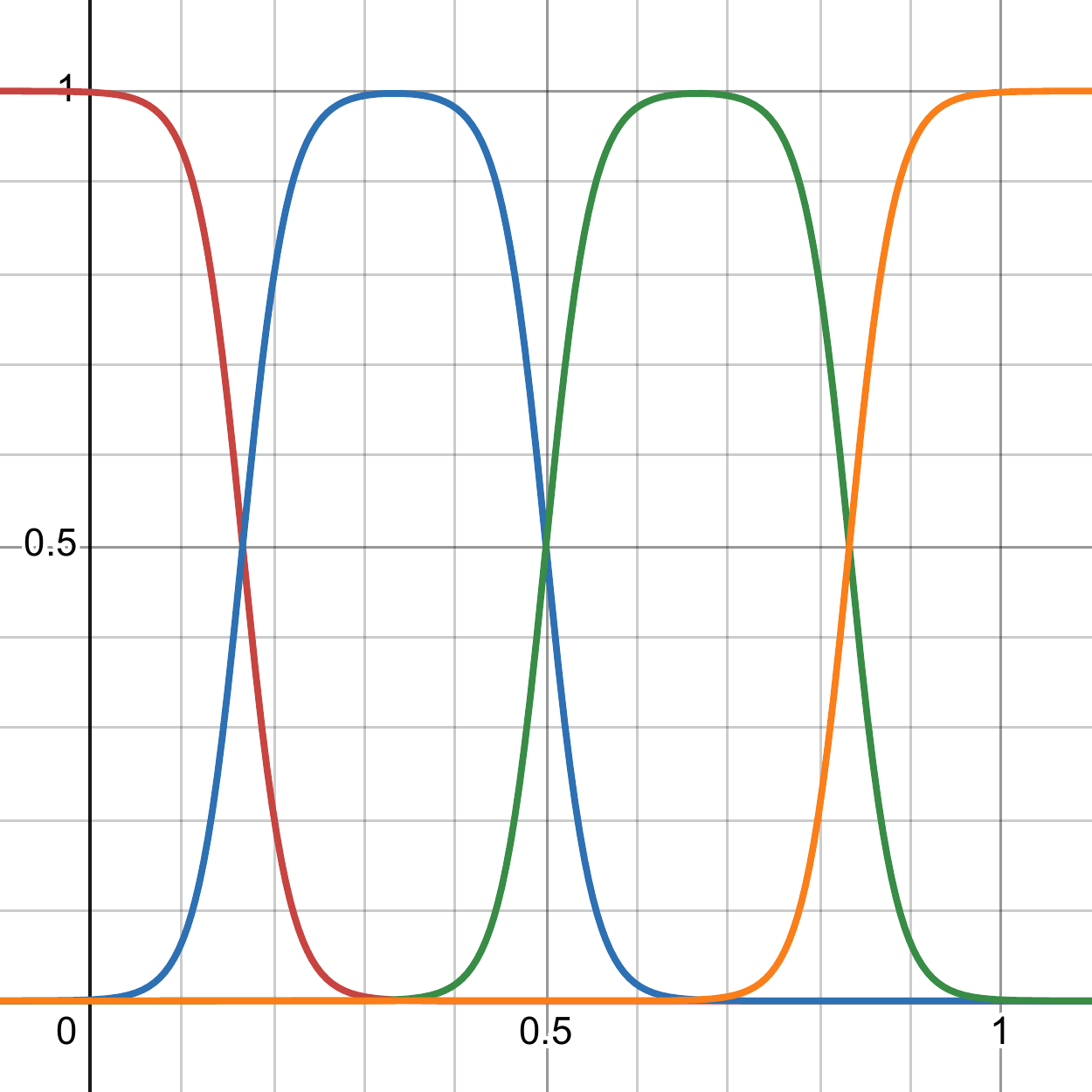}
	\caption{Example partition of unity $\{\psi_k\}_{k=0}^N$ for $N = 3$ and $c = 20$.}
	\label{fig:ac}
\end{figure}

\begin{proof}
	To simplify the subscript notation, define
	\begin{align*}
		\alpha_j &:= y_{j} \\
		\lambda_{j,k} &:= y_{d + (k+1)j} \\
		\beta_j &:= y_{(N+2)d + j} \\
		\mu_{j,k} &:= y_{(N+3)d + (k+1)j}.
	\end{align*}
	In the following steps, any omitted ODEs will have right-hand sides identically equal to zero. For each step, let $j=1,\dots,d$ and unless otherwise specified let $k=0,\dots,N$.
	
		\noindent \textbf{Step 1} (compute $\tanh$): For times $t \in [0,1)$, let $\dot{\lambda}_{j,0} = 1$ and for $k=1,\dots,N$ let
	\begin{align*}
		\dot{\lambda}_{j,k} &= c \left( \alpha_j - \frac{2k-1}{2N} \right) - \lambda_{j,k} \mu_{j,k} \\
		\dot{\mu}_{j,k} &= c^2 \left( \alpha_j - \frac{2k-1}{2N} \right)^2 - \mu_{j,k}^2.
	\end{align*}
	Then by Lemma \ref{lem:tanh} we have
	\begin{align*}
		\alpha_j(1) &= x_j \\
		\lambda_{j,0}(1) &= 1 \\
		\lambda_{j,k}(1) &= \tanh \left( c \left( x_j - \frac{2k-1}{2N} \right) \right) \\
		\beta_j(1) &= 0 \\
		\mu_{j,0}(1) &= 0 \\
		\mu_{j,k}(1) &= c \left( x_j - \frac{2k-1}{2N} \right) \tanh \left( c \left( x_j - \frac{2k-1}{2N} \right) \right).
	\end{align*}
	\textbf{Step 2} (compute $\psi_k$): For times $t \in [1,2)$ let
	\begin{equation*}
		\dot{\lambda}_{j,k} = \sum_{\ell=0}^N a_{k,\ell} \lambda_{j,\ell} + b_k
	\end{equation*}
	and for $k=1,\dots,N$ let
	\begin{equation*}
		\dot{\mu}_{j,k} = c \left( \alpha_j - \frac{2k-1}{2N} \right) \sum_{\ell=0}^N a_{k-1,\ell} \lambda_{j,\ell},
	\end{equation*}
	where $a_{k,\ell}$ is the element in the $k+1$-th row and $\ell+1$-th column of the matrix $A \in \Reals^{N+1 \times N+1}$ given by
	\begin{equation*}
		A = \begin{bmatrix} -\ln(2) & -1 & -\frac{1}{2} & \hdots & -\frac{1}{N} \\
			0 & -\ln(2) & -1 & \hdots & -\frac{1}{N-1} \\
			0 & 0 & -\ln(2) & \hdots & -\frac{1}{N-2} \\
			\vdots & \vdots & \vdots & \ddots & \vdots \\
			0 & 0 & 0 & \hdots & -\ln(2) \end{bmatrix},
	\end{equation*}
	and $b_k$ is the $k+1$-th element of the vector
	\begin{equation*}
		b = \left( \int_1^2 e^{(2-t) A} \dif t \right)^{-1} \begin{bmatrix} 0 & \hdots & 0 & \frac{1}{2} \end{bmatrix}^\Tr \in \Reals^{N+1}.
	\end{equation*}
	It can be verified using e.g., a computer algebra system that
	\begin{equation*}
		e^A = \begin{bmatrix} \frac{1}{2} & -\frac{1}{2} & 0 & \hdots & 0 \\
			0 & \frac{1}{2} & -\frac{1}{2} & \hdots & 0 \\
			0 & 0 & \frac{1}{2} & \hdots & 0 \\
			\vdots & \vdots & \vdots & \ddots & \vdots \\
			0 & 0 & 0 & \hdots & \frac{1}{2} \end{bmatrix}.
	\end{equation*}
	Notice that $A$ is invertible because it is upper triangular and its diagonal entries are all nonzero. The inverse of the matrix $\int_0^1 e^{(1-t) A} \dif t$ is given by
	\begin{align*}
		\left( \int_0^1 e^{(1-t) A} \dif t \right)^{-1} &= \left( e^A \int_0^1 e^{-t A} \dif t \right)^{-1} \\
		&= \left( e^A \left[ -A^{-1} e^{-t A} \right]_0^1 \right)^{-1} \\
		&= \left( e^A \left( -A^{-1} e^{-A} + A^{-1} \right) \right)^{-1} \\
		&= \left( \left( e^A - I \right) A^{-1} \right)^{-1} \\
		&= A \left( e^A - I \right)^{-1}.
	\end{align*}
	The matrix $e^A - I$ is invertible because $e^A$ is upper triangular and its diagonal entries are all equal to $\frac{1}{2}$, hence $e^A - I$ is also upper triangular and its diagonal entries are all equal to $-\frac{1}{2}$ (in particular, they are all nonzero), therefore $b$ is well-defined. Then we have
	\begin{equation*}
		\lambda_{j,0:N}(2) = e^A \lambda_{j,0:N}(1) + \left( \int_1^2 e^{(2-t) A} \dif t \right) b,
	\end{equation*}
	which can be evaluated to obtain the solution after Step 2 as
	\begin{align*}
		\alpha_j(2) &= x_j \\
		\lambda_{j,k}(2) &= \psi_k(x_j) \\
		\beta_j(2) &= 0 \\
		\mu_{j,k}(2) &= 0.
	\end{align*}
	\textbf{Step 3} (setup for $\ln$): For times $t \in [2,3)$ and for $k=0,\dots,N-1$ let
	\begin{align*}
		\dot{\alpha}_j &= \Delta - 1 \\
		\dot{\lambda}_{j,k} &= - 1,
	\end{align*}
	so we have
	\begin{align*}
		\alpha_j(3) &= x_j + \Delta - 1 \\
		\lambda_{j,k}(3) &= \psi_k(x_j) - 1 \\
		\beta_j(3) &= 0 \\
		\mu_{j,k}(3) &= 0.
	\end{align*}
	Step 4 (compute $\ln$): For times $t \in [3,4)$ let
	\begin{align*}
		\dot{\alpha}_j &= -\alpha_j^2 \\
		\dot{\lambda}_{j,k} &= -\lambda_{j,k}^2 \\
		\dot{\beta}_j &= \alpha_j \\
		\dot{\mu}_{j,k} &= \lambda_{j,k}.
	\end{align*}
	Then since $x_j + \Delta > 0$, by Lemma \ref{lem:ln} we have
	\begin{align*}
		\alpha_j(4) &= \frac{x_j + \Delta - 1}{x_j + \Delta} \\
		\lambda_{j,k}(4) &= \frac{\psi_k(x_j) - 1}{\psi_k(x_j)} \\
		\beta_j(4) &= \ln(x_j + \Delta) \\
		\mu_{j,k}(4) &= \ln(\psi_k(x_j)).
	\end{align*}
	\textbf{Steps 5-7} (cleanup): For times $t \in [4,5)$ let
	\begin{align*}
		\dot{\alpha}_j &= \beta_j \alpha_j \\
		\dot{\lambda}_{j,k} &= \mu_{j,k} \lambda_{j,k}.
	\end{align*}
	Both $x_j + \Delta > 0$ and $\psi_k(x_k) > 0$ hold, so by Lemma \ref{lem:multiply} we have
	\begin{align*}
		\alpha_j(5) &= x_j + \Delta - 1 \\
		\lambda_{j,k}(5) &= \psi_k(x_j) - 1 \\
		\beta_j(5) &= \ln(x_j + \Delta ) \\
		\mu_{j,k}(5) &= \ln(\psi_k(x_j)).
	\end{align*}
	Now for times $t \in [5,6)$ let
	\begin{align*}
		\dot{\alpha}_j &= 1 \\
		\dot{\lambda}_{j,k} &= 1,
	\end{align*}
	so we have
	\begin{align*}
		\alpha_j(6) &= x_j + \Delta \\
		\lambda_{j,k}(6) &= \psi_k(x_j) \\
		\beta_j(6) &= \ln(x_j + \Delta) \\
		\mu_{j,k}(6) &= \ln(\psi_k(x_j)).
	\end{align*}
	Finally for times $t \in [6,7]$ let
	\begin{align*}
		\dot{\alpha}_j &= -\beta_j \alpha_j \\
		\dot{\lambda}_{j,k} &= -\mu_{j,k} \lambda_{j,k},
	\end{align*}
	so we have
	\begin{align*}
		\alpha_j(7) &= 1 \\
		\lambda_{j,k}(7) &= 1 \\
		\beta_j(7) &= \ln(x_j + \Delta) \\
		\mu_{j,k}(7) &= \ln(\psi_k(x_j)),
	\end{align*}
	which is the desired terminal state. For the reader's convenience, Table \ref{tab:bootstrap_summary} summarizes the state after each step.
\end{proof}

\begin{table*}[t]
	\centering
	\caption{Summary of the state after each step (duration-1 control segment) in the proof Lemma \ref{lem:bootstrap} (bootstrapping). Here, indices range from $j = 1,\dots,d$ and $k = 1,\dots,N$, and in Step 1, $\gamma_{j,k} = c\left(x_j - \frac{2k-1}{2N}\right)$.}
	{\footnotesize
	\begin{tabular}{|c|c c c c c c c c|}
		\hline
		& Step 0 & Step 1 & Step 2 & Step 3 & Step 4 & Step 5 & Step 6 & Step 7 \\
		\hline
		$\alpha_j$      & $x_j$ & $x_j$ & $x_j$ & $x_j + \Delta - 1$ & $\frac{x_j + \Delta - 1}{x_j + \Delta}$ & $x_j + \Delta - 1$ & $x_j + \Delta$ & 1 \\
		$\lambda_{j,0}$ & 0 & 1 & $\psi_0(x_j)$ & $\psi_0(x_j) - 1$ & $\frac{\psi_0(x_j) - 1}{\psi_0(x_j)}$ & $\psi_0(x_j) - 1$ & $\psi_0(x_j)$ & 1 \\
		$\lambda_{j,k}$ & 0 & \hspace{-5mm} $\tanh(\gamma_{j,k})$ \hspace{-5mm} & $\psi_k(x_j)$ & $\psi_k(x_j) - 1$ & $\frac{\psi_k(x_j) - 1}{\psi_k(x_j)}$ & $\psi_k(x_j) - 1$ & $\psi_k(x_j)$ & 1 \\
		$\beta_j$       & 0 & 0 & 0 & 0 & $\ln(x_j + \Delta)$ & $\ln(x_j + \Delta)$ & $\ln(x_j + \Delta)$ & $\ln(x_j + \Delta)$ \\
		$\mu_{j,0}$     & 0 & 0 & 0 & 0 & $\ln(\psi_0(x_j))$ & $\ln(\psi_0(x_j))$ & $\ln(\psi_0(x_j))$ & $\ln(\psi_0(x_j))$ \\
		$\mu_{j,k}$     & 0 & \hspace{-5mm} $\gamma_{j,k} \tanh(\gamma_{j,k})$ \hspace{-5mm} & 0 & 0 & $\ln(\psi_k(x_j))$ & $\ln(\psi_k(x_j))$ & $\ln(\psi_k(x_j))$ & $\ln(\psi_k(x_j))$ \\
		\hline
	\end{tabular}}
	\label{tab:bootstrap_summary}
\end{table*}

Recall that the \textit{Sobolev space} $\cW^{k,p}(\Omega)$, where $\Omega \subseteq \Reals^d$, is defined as the set of functions $f : \Omega \to \Reals$ such that the \textit{Sobolev norm}
\begin{equation*}
	\|f\|_{k,p} := \left( \sum_{0 \leq |\alpha| \leq k} \left\| f^{(\alpha)} \right\|_p^p \right)^{\frac{1}{p}} = \left( \sum_{0 \leq |\alpha| \leq k} \int_\Omega |f^{(\alpha)}(x)|^p \dif x \right)^{\frac{1}{p}}
\end{equation*}
is finite. Here
\begin{equation*}
	f^{(\alpha)} = \frac{\partial^{|\alpha|} f}{\partial x_1^{\alpha_1} \cdots \partial x_d^{\alpha_d}}
\end{equation*}
is the mixed partial derivative with orders specified by the multi-index $\alpha = (\alpha_1,\dots,\alpha_d)$, where $\alpha_i \geq 0$ for all $i=1,\dots,d$, and $|\alpha| := \alpha_1 + \dots + \alpha_d$. Now we are ready to state and prove the first main result.

\begin{theorem}\label{thm:sobolev}
	Fix any ${\epsilon > 0}$ and consider any function ${f \in \cW^{n,\infty}([0, 1]^d)}$ with Sobolev norm $\|f\|_{n,\infty} \leq 1$. Then there exists an approximation $\hat{f} : [0,1]^d \to \Reals$ given by a linear read-out of the terminal state of a quadratic neural ODE with $W = \cO\left( \epsilon^{-\frac{1}{n}} \right)$ states and $D = \cO\left( \epsilon^{-\frac{d-1}{n}} \right)$ piecewise-constant control segments such that
	\begin{equation*}
		\sup_{x \in [0,1]^d} |f(x) - \hat{f}(x)| \leq \epsilon.
	\end{equation*}
\end{theorem}

\begin{proof}
	Consider the partition of unity $\{ \psi_k : [0, 1] \to (0,1) \}_{k \in \{0,\dots,N\}}$ of the unit interval $[0,1] \subset \Reals$ described in Lemma \ref{lem:bootstrap}, parametrized by the positive integer $N \geq 1$ and positive real number $c > 0$. For each $k=1,\dots,N-1$ we define the \textit{approximate support} of $\psi_k$ to be
	\begin{equation*}
		\sS_k := \left[ \frac{k-1}{N}, \frac{k+1}{N} \right],
	\end{equation*}
	with $\sS_0 := [ 0, 1/N ]$ and $\sS_N := [ (N-1)/N, 1 ]$ being the approximate supports of $\psi_0$ and $\psi_N$, respectively. For a given $0 < \delta < \frac{1}{2}$, choose $c \geq 2N \tanh^{-1}(1 - 2 \delta)$ so that
	\begin{equation*}
		\sup \left\{ \psi_k(x) : 0 \leq k \leq N,\ x \in [0, 1] \setminus \sS_k \right\} \leq \delta.
	\end{equation*}
	In other words, $\psi_k(x) \leq \delta$ for any $x \notin \sS_k$. We can form a partition of unity $\{ \phi_{\bk} : [0, 1]^d \to (0,1) \}_{\bk \in \{0,\dots,N\}^d}$
	of the unit cube $[0,1]^d \subset \Reals^d$ by setting $\phi_{\bk}(x) := \psi_{k_1}(x_1) \cdots \psi_{k_d}(x_d)$, where $\bk$ denotes the tuple $(k_1,\dots,k_d)$. The element $\phi_{\bk}$ has approximate support $\sS_{\bk} := \sS_{k_1} \times \cdots \times \sS_{k_d}$. Following \cite{Yarotsky_2017}, we define a local Taylor approximation of $f$ by
	\begin{equation*}
		\hat{f}(x) = \sum_{\bk \in \{0,\dots,N\}^d} \phi_{\bk}(x) P_{\bk}(x),
	\end{equation*}
	where $P_{\bk} : [0,1]^d \to \Reals$ is the degree $n-1$ Taylor polynomial of $f$ at $\left(\frac{k_1}{N},\dots,\frac{k_d}{N} \right)$, given by
	\begin{equation*}
	\begin{split}
		P_{\bk}(x) &= \sum_{n_1+\dots+n_d < n} \frac{\partial^{n_1}}{\partial {x_1}^{n_1}} \cdots \frac{\partial^{n_d}}{\partial {x_d}^{n_d}} f\left( \frac{k_1}{N},\dots,\frac{k_d}{N} \right) \\
		&\hspace{1cm} \cdot \prod_{j=1}^d \frac{1}{n_j!} \left( x_j - \frac{k_j}{N} \right)^{n_j}.
	\end{split}
	\end{equation*}
	The error of the local Taylor approximation is bounded by
	\begin{align*}
		&\phantom{{}={}} |f(x) - \hat{f}(x)| \\
		&= \Bigg| f(x) - \sum_{\bk \in \{0,\dots,N\}^d} \phi_{\bk}(x) P_{\bk}(x) \Bigg| \\
		&\leq \sum_{\bk : x \in \sS_{\bk}} \left| f(x) - P_{\bk}(x) \right| + \sum_{\bk : x \notin \sS_{\bk}} \delta \left| f(x) - P_{\bk}(x) \right| \\
		&\leq 2^d \|f\|_{n,\infty} \frac{d^n}{n!} \left( \frac{1}{N} \right)^n + \left( (N+1)^d - 2^d \right) \delta \|f\|_{n,\infty} \frac{d^n}{n!} (1)^n \\
		&\leq \frac{d^n}{n!} \left( 2^d \left( \frac{1}{N} \right)^n + \left( (N+1)^d - 2^d \right) \delta \right).
	\end{align*}
	The factor $2^d$ comes from the number of multi-indices $\bk$ such that $x \in \sS_{\bk}$, since along each dimension a point $x \in [0,1]$ can only be in at most 2 approximate supports $\sS_k$. The factor $(N+1)^d - 2^d$ is then the number of multi-indices $\bk$ such that $x \not\in \sS_{\bk}$, which follows from the previous sentence and the fact that there are $(N+1)^d$ total functions in the partition of unity of $[0,1]^d$ described above. The rest of the expression follows from standard Taylor remainder bounds, combined with the multinomial theorem where applicable. Now we can choose $N$ and $\delta$ according to
	\begin{equation*}
		N = \left\lceil \left( \frac{n!}{d^n} \frac{\gamma \epsilon}{2^d} \right)^{-1/n} \right\rceil,\quad\quad \delta \leq \frac{n!}{d^n} \frac{(1 - \gamma) \epsilon}{\left(N + 1 \right)^d - 2^d}
	\end{equation*}
	for some arbitrary $\gamma \in (0,1)$, so that
	\begin{equation*}
		\sup_{x \in [0,1]^d} |f(x) - \hat{f}(x)| \leq \epsilon.
	\end{equation*}
	Now it remains to show how the approximation $\hat{f}$ can be generated by a quadratic neural ODE. We will use
	\begin{align*}
		W &= 2(N+2)d+1 = \cO\left( \epsilon^{-\frac{1}{n}} \right) \\
		D &= 7 + \left\lceil \frac{2 M (N+1)^d}{(N+2)d} \right\rceil = \cO\left( \epsilon^{-\frac{d-1}{n}} \right)
	\end{align*}
	states and piecewise-constant control segments, respectively, where $M = {d+n-1 \choose d}$ is the number of monomials in $d$ variables with degree at most $n-1$. First, fix any $\Delta > 0$ and define $M (N+1)^d$ new coefficients
	\begin{equation*}
		\left\{ a_{k_1,\dots,k_d}^{n_1,\dots,n_d} : n_1 + \cdots + n_d < n,\ \bk \in \{0,\dots,N\}^d \right\} \subset \Reals,
	\end{equation*}
	which depend on $\Delta$ but do not depend on $x_1,\dots,x_d$, so that the polynomial $P_{\bk}(x)$ can be represented as
	\begin{equation*}
		P_{\bk}(x) = \sum_{n_1+\dots+n_d<n} a_{k_1,\dots,k_d}^{n_1,\dots,n_d} (x_1+\Delta)^{n_1} \cdots (x_d+\Delta)^{n_d}
	\end{equation*}
	for all $x \in [0,1]^d$. The constant $\Delta > 0$ is added to each $x_1,\dots,x_d$ for the reason of satisfying the assumption $\xi > 0$ in Lemma \ref{lem:ln}. Now consider a quadratic neural ODE with $W$ states $y_{1:W}$ initialized at $y_{1:d}(0) = x$ and $y_{d+1:W} = 0$. Apply Lemma \ref{lem:bootstrap} to bootstrap the ODE to the state
	\begin{align*}
		y_j(7) &= 1 \\
		y_{d + (k+1)j}(7) &= 1 \\
		y_{(N+2)d + j}(7) &= \ln(x_j + \Delta) \\
		y_{(N+3)d + (k+1)j}(7) &= \ln(\psi_k(x_j)) \\
		y_{2(N+2)d + 1}(7) &= 0
	\end{align*}
	using 7 duration-1 control segments, where $j=1,\dots,d$ and $k=0,\dots,N$. Now we will repeatedly apply Lemma \ref{lem:multiply} to construct each of the $M (N+1)^d$ terms in the local Taylor approximation. A prototypical term takes the form
	\begin{equation*}
		a_{k_1,\dots,k_d}^{n_1,\dots,n_d} (x_1+\Delta)^{n_1} \cdots (x_d+\Delta)^{n_d} \psi_{k_1}(x_1) \cdots \psi_{k_d}(x_d)
	\end{equation*}
	for some $n_1,\dots,n_d,k_1,\dots,k_d$. Assume that $a_{k_1,\dots,k_d}^{n_1,\dots,n_d} \neq 0$, otherwise we can simply skip that term. For the first step after the bootstrapping phase (i.e., the 8th control segment), consider Lemma \ref{lem:multiply} with $y_1,\dots,y_m$ replaced by $y_{(N+2)d + 1},\dots,y_{2(N+2)d}$ and $y_{m+1}$ replaced by $y_1$. Let $w_{(N+2)d + j} \gets n_j$ and $w_{(N+3)d + (k_j+1)j} \gets 1$ for $j=1,\dots,d$ and let $b \gets \ln\left( \left| a_{k_1,\dots,k_d}^{n_1,\dots,n_d} \right| \right)$. This yields
	\begin{equation*}
	\begin{split}
		y_1(8) &= \left| a_{k_1,\dots,k_d}^{n_1,\dots,n_d} \right| (x_1+\Delta)^{n_1} \cdots (x_d+\Delta)^{n_d} \\
		&\hspace{1cm} \cdot \psi_{k_1}(x_1) \cdots \psi_{k_d}(x_d),
	\end{split}
	\end{equation*}
	which is precisely the term we wished to construct (modulo a sign). There are $(N+2)d$ available states similar to $y_1$ which can construct similar terms in parallel within the same duration-1 control segment. During the next step (i.e., the 9th control segment) we can add these terms to the final state which holds a running sum via
	\begin{equation*}
		\dot{y}_{2(N+2)d + 1} = \text{sign}\left( a^{(1)} \right) y_1 + \cdots + \text{sign}\left( a^{(N+2)d} \right) y_{(N+2)d},
	\end{equation*}
	where $a^{(i)} = a_{k_1,\dots,k_d}^{n_1,\dots,n_d}$ for some $n_1,\dots,n_d,k_1,\dots,k_d$ (since the signs of the coefficients were dropped in order for $b$ to be well-defined, we need to reintroduce them when adding the terms to the running sum). Now we can repeat the process, except that in the subsequent steps we need to also divide by the previous terms so they are not added to the running sum again. We can accomplish this by letting
	\begin{align*}
		w_{(N+2)d + j_{\text{new}}} &\gets n_{j_{\text{new}}} - n_{j_{\text{old}}} \\
		w_{(N+3)d + (k_{j_{\text{old}}}+1) j_{\text{old}}} &\gets -1 \\
		w_{(N+3)d + (k_{j_{\text{new}}}+1) j_{\text{new}}} &\gets 1 \\
		b_{\text{new}} &\gets \ln\left( \left| a_{k_1,\dots,k_d}^{n_1,\dots,n_d} \right| \right) - b_{\text{old}}
	\end{align*}
	for each new term. Each iteration of two control segments --- one to construct the terms and one to add them to the running sum --- takes care of $(N+2)d$ terms. Since there are $M (N+1)^d$ total terms, this requires
	\begin{equation*}
		D = \left\lceil \frac{2 M (N+1)^d}{(N+2)d} \right\rceil
	\end{equation*}
	total control segments in addition to the first 7 used during the bootstrapping phase. The final approximation is given by the terminal value of the running sum state
	\begin{equation*}
		\hat{f}(x) = y_{2(N+2)d + 1}(D),
	\end{equation*}
	which completes the proof. If we have a desired time horizon $T$ that is different from $D$, we can rescale the right-hand sides of every ODE in the construction by $D/T$, because each control segment is individually time-invariant.
\end{proof}

We have shown that quadratic neural ODEs can approximate functions in a Sobolev ball. Now we will show that quadratic neural ODEs can approximate deep feedforward nets with $\tanh$ activation functions. The proof can be adapted to other activation functions that can be generated or approximated by the solution to some system of quadratic ODEs; some examples are given at the end of this section. For $\tanh$ activation functions generated using the system of quadratic ODEs in Lemma \ref{lem:tanh}, we will need the following lemma:

\begin{lemma}\label{lem:tanh_perturbation}
	Consider the system of quadratic ODEs in Lemma \ref{lem:tanh} with perturbed initial conditions
	\begin{alignat*}{2}
		\dot{y}_1 &= 0,\quad & y_1(0) &= \xi + \eta_1 \\
		\dot{y}_2 &= y_1 - y_2 y_3,\quad & y_2(0) &= \eta_2 \\
		\dot{y}_3 &= y_1^2 - y_3^2,\quad & y_3(0) &= \eta_3
	\end{alignat*}
	where $|\xi| \leq K$ and $\max\{|\eta_1|,|\eta_2|,|\eta_3|\} \leq \delta$. For a given perturbation $\eta = (\eta_1,\eta_2,\eta_3)$, denote the resulting time-$t$ solution by $y(t; \eta)$. For sufficiently small $\delta > 0$, we have
	\begin{equation*}
		\| y(1; \eta) - y(1; 0) \|_\infty \leq \frac{(4K+2) e^{4K+2} \delta}{4K+2 - 9 (e^{4K+2}-1) \delta}.
	\end{equation*}
\end{lemma}

\begin{proof}
	The Jacobian of $f(y) = (0, y_1 - y_2 y_3, y_1^2 - y_3^2)$ evaluated along the unperturbed trajectory $y(t; 0)$ is given by
	\begin{align*}
		\frac{\partial f}{\partial y}\Big|_{y(t; 0)} &= \begin{bmatrix} 0 & 0 & 0 \\ 1 & -y_3(t; 0) & -y_2(t; 0) \\ 2 y_1(t; 0) & 0 & -2 y_3(t; 0) \end{bmatrix} \\
		&= \begin{bmatrix} 0 & 0 & 0 \\ 1 & -\xi \tanh(\xi t) & -\tanh(\xi t) \\ 2 \xi & 0 & -2 \xi \tanh(\xi t) \end{bmatrix}.
	\end{align*}
	The $\ell^\infty$ operator norm of the Jacobian evaluated along $y(t; 0)$ is given by the maximum $\ell^1$ norm over its rows:
	\begin{align*}
		\smash{\left\| \frac{\partial f}{\partial y}\Big|_{y(t; 0)} \right\|_\infty}
		&\leq \max\{ 0, \\
		&\quad\quad 1 + |-\xi \tanh(\xi t)| + |-\tanh(\xi t)|, \\
		&\quad\quad |2 \xi| + 0 + |-2 \xi \tanh(\xi t)| \} \\
		&\leq 4K+2.
	\end{align*}
	The second-order derivatives evaluated along $y(t; 0)$ are bounded by
	\begin{equation*}
		\left\| \frac{\partial^2 f}{\partial y_i \partial y_j}\Big|_{y(t; 0)} \right\|_\infty \leq 2.
	\end{equation*}
	Define $\delta y = (\delta y_1, \delta y_2, \delta y_3)$ where $\delta y_i(t) = y_i(t; \eta) - y_i(t; 0)$ for $i=1,2,3$. By Taylor's remainder theorem, there exists a function ${\lambda : [0,1] \to [0,1]}$ and a path $\zeta(t) = \lambda(t) y(t; \eta) + (1-\lambda(t)) y(t; 0)$ such that
	\begin{equation*}
		\frac{\dif}{\dif t} \delta y = \frac{\partial f}{\partial y}\Big|_{y(t; 0)} \delta y + \frac{1}{2} \sum_{i,j=1}^3 \frac{\partial^2 f}{\partial y_i \partial y_j}\Big|_{\zeta(t)} \delta y_i \delta y_j.
	\end{equation*}
	Now we have
	\begin{align*}
		\frac{\dif}{\dif t} \| \delta y \|_\infty &\leq \left\| \frac{\dif}{\dif t} \delta y \right\|_\infty \\
		&= \left\| \frac{\partial f}{\partial y}\Big|_{y(t; 0)} \delta y + \frac{1}{2} \sum_{i,j=1}^3 \frac{\partial^2 f}{\partial y_i \partial y_j}\Big|_{\zeta(t)} \delta y_i \delta y_j \right \|_\infty \\
		\begin{split}
			&\leq \left\| \frac{\partial f}{\partial y}\Big|_{y(t; 0)} \right\|_\infty \| \delta y \|_\infty \\
			&\quad\quad + \frac{1}{2} \sum_{i,j=1}^3 \left\| \frac{\partial^2 f}{\partial y_i \partial y_j}\Big|_{\zeta(t)} \right\|_\infty |\delta y_i \delta y_j|
		\end{split} \\
		&\leq (4K+2) \| \delta y \|_\infty + 9 \| \delta y \|_\infty^2.
	\end{align*}
	Define $w(u) = (4K+2) u + 9 u^2$ and
	\begin{align*}
		G(x) &= \int_{\delta}^x \frac{\dif u}{w(u)} \\
		&= \int_{\delta}^x \frac{\dif u}{(4K+2) u + 9 u^2} \\
		&= \frac{1}{4K+2} \ln \left( \frac{x}{\delta} \right) - \frac{1}{4K+2} \ln \left( \frac{4K+2 + 9x}{4K+2 + 9\delta} \right).
	\end{align*}
	Since $\| \delta y(0) \|_\infty \leq \delta$, by the Bihari-Lasalle inequality \cite{Lasalle_1949}, \cite{Bihari_1956} we have
	\begin{equation*}
		\| \delta y(1) \|_\infty \leq G^{-1}\left( G(\delta) + \int_0^1 1 \dif t \right) = G^{-1}(1).
	\end{equation*}
	If $\delta < \frac{4K+2}{9 (e^{4K+2}-1)}$, then we have
	\begin{equation*}
		\| \delta y(1) \|_\infty \leq \frac{(4K+2) e^{4K+2} \delta}{4K+2 - 9 (e^{4K+2}-1) \delta},
	\end{equation*}
	which completes the proof.
\end{proof}

Now we will state and prove the feedforward net approximation result.

\begin{theorem}\label{thm:feedforward}
	Let $\epsilon > 0$ and consider a feedforward neural net $f : [0,1]^d \to \Reals$ with width $\widetilde{W}$ and depth $\widetilde{D}$ that has $\tanh$ activation functions. Assume that any particular weight is bounded in absolute value by $M > 0$. Then there exists an approximation $\hat{f} : [0,1]^d \to \Reals$ given by a linear read-out of the terminal state of a quadratic neural ODE with ${W = 3 \max\{\widetilde{W}, d\}}$ states and ${D = 3 \widetilde{D}}$ piecewise-constant control segments such that
	\begin{equation*}
		\sup_{x \in [0,1]^d} |f(x) - \hat{f}(x)| \leq \epsilon.
	\end{equation*}
	Remark: $W$ and $D$ are independent of $\epsilon$ and $M$.
\end{theorem}

\begin{proof}
	For convenience we will assume that $\widetilde{W} = d$; adapting the construction to the case where {$\widetilde{W} \neq d$} only requires minor changes. Let $z(k) \in \Reals^{\widetilde{W}}$ represent the outcome of applying the first $k$ layers of the feedforward net to the input $z(0) = x$. Then we have the recurrence relation
	\begin{equation*}
		z(k+1) = \tanh(A_k z(k) + b_k),
	\end{equation*}
	where ${A_k \in \Reals^{\widetilde{W} \times \widetilde{W}}}$ and ${b_k \in \Reals^{\widetilde{W}}}$ for ${k=1,\dots,\widetilde{D}}$. The final output is given by ${f(\xi) = c^\Tr z(\widetilde{D})}$. To simplify the subscript notation, we will again define
	\begin{align*}
		\alpha_j &:= y_j \\
		\lambda_j &:= y_{\widetilde{W} + j} \\
		\mu_j &:= y_{2\widetilde{W} + j}
	\end{align*}
	for $j=1,\dots,\widetilde{W}$ as done in the proof of Lemma \ref{lem:bootstrap}. 

	In the following steps, we will construct an approximation to the $k$th layer using $3$ control segments. Suppose that for each $k \geq 1$, there exist $\delta_k > 0$ and $\xi,\eta_1,\eta_2,\eta_3 \in \Reals^{\widetilde{W}}$ with $\|\xi\|_\infty \leq 1$ and $\max\{\|\eta_1\|_\infty, \|\eta_2\|_\infty, \|\eta_3\|_\infty\} \leq \delta_k$ such that
	\begin{alignat*}{2}
		\alpha(3k) &= \xi + \eta_1 \\
		\lambda(3k) &= \eta_2 \\
		\mu(3k) &= \eta_3.
	\end{alignat*}
	Evidently, this holds at $k=0$ (where any $\delta_0 > 0$ will suffice), so it remains to show that for subsequent layers we can control $\delta_{k+1}$ in terms of $\delta_k$.

		\noindent \textbf{Step 1} (weights and biases): Let
	\begin{align*}
		\dot{\alpha} &= -r_1 \alpha \\
		\dot{\lambda} &= 0 \\
		\dot{\mu} &= \frac{r_1}{1-e^{-r_1}} A_k \alpha + b_k,
	\end{align*}
	where $r_1 = -\ln \left( \frac{(M \widetilde{W}+1) \delta_k}{1+\delta_k} \right)$. Then we have
	\begin{align*}
		\alpha(3k+1) &= e^{-r_1} \left( \xi + \eta_1 \right) = \frac{(M \widetilde{W}+1) \delta_k}{1+\delta_k} \left( \xi + \eta_1 \right) \\
		\lambda(3k+1) &= \eta_2 \\
		\mu(3k+1) &= \left(A_k \xi + b_k\right) + \left(A_k \eta_1 + \eta_3 \right).
	\end{align*}
	Since $\|\xi\|_\infty \leq 1$,
	\begin{align*}
		\max\{ &\|\alpha(3k+1)\|_\infty, \\
		&\|\lambda(3k+1)\|_\infty, \\
		&\|\mu(3k+1) - (A_k \xi + b_k)\|_\infty \} \leq (M \widetilde{W}+1) \delta_k.
	\end{align*}

		\noindent \textbf{Step 2} ($\tanh$): Let $K = M(\widetilde{W}+1)$ and
	\begin{align*}
		\dot{\alpha}_j &= \mu_j - \alpha_j \lambda_j \\
		\dot{\lambda}_j &= \mu_j^2 - \lambda_j^2 \\
		\dot{\mu}_j &= 0
	\end{align*}
	for $j=1,\dots,\widetilde{W}$. Then provided that $\delta_k < \frac{4K+2}{9 (e^{4K+2}-1) K}$, by Lemma \ref{lem:tanh_perturbation} we have
	\begin{align*}
		\max \{ &\| \alpha(3k+2) - \tanh(A_k \xi + b_k) \|_\infty, \\
		&\| \lambda(3k+2) - (A_k \xi + b_k) \tanh(A_k \xi + b_k) \|_\infty, \\
		&\| \mu(3k+2) - (A_k \xi + b_k) \|_\infty\} \leq \Delta_k,
	\end{align*}
	where
	\begin{equation*}
		\Delta_k := \frac{(4K+2) e^{4K+2} K \delta_k}{4K+2 - 9 (e^{4K+2}-1) K \delta_k} = \frac{a_1 \delta_k}{a_3 - a_2 \delta_k}
	\end{equation*}
	for ${a_1 = a_3 e^{a_3} K}$, ${a_2 = 9 (e^{a_3}-1) K}$, and ${a_3 = 4K+2}$.

		\noindent \textbf{Step 3} (cleanup): Let
	\begin{align*}
		\dot{\alpha} &= 0 \\
		\dot{\lambda} &= -r_3 \lambda \\
		\dot{\mu} &= -r_3 \mu,
	\end{align*}
	where $r_3 = -\ln \left(\frac{\Delta_k}{K+\Delta_k} \right)$. Then since $\| A_k \xi + b_k \|_\infty \leq K$,
	\begin{align*}
		\max \{ &\| \alpha(3k+3) - \tanh(A_k \xi + b_k) \|_\infty, \\
		&\| \lambda(3k+3) \|_\infty, \\
		&\| \mu(3k+3) \|_\infty\} \leq \Delta_k.
	\end{align*}

	We have shown that $\delta_{k+1} \leq \Delta_k$. Composing this bound $\widetilde{D}$ times yields
	\begin{equation*}
		\delta_{\widetilde{D}} \leq \frac{a_1^{\widetilde{D}} \delta_0}{a_3^{\widetilde{D}} - a_2 \left( a_1 + a_3 \right)^{\widetilde{D}} \delta_0}
	\end{equation*}
	The final output of the quadratic neural ODE is given by $\hat{f}(x) = c^\Tr \alpha(3\widetilde{D})$, so taking
	\begin{equation*}
		\delta_0 \leq \frac{a_3^{\widetilde{D}} \epsilon}{a_1^{\widetilde{D}} \|c^\Tr\|_1 + a_2 \left( a_1 + a_3 \right)^{\widetilde{D}} \epsilon}
	\end{equation*}
	yields $\delta_{\widetilde{D}} \leq \epsilon / \|c^\Tr\|_1$, which completes the proof. If not satisfied already, the condition ${\delta_k < \frac{4K+2}{9 (e^{4K+2}-1) K}}$ in Step 2 will hold by taking $\delta_0$ to be sufficiently small.
\end{proof}

The construction in the proof of Theorem \ref{thm:feedforward} can be adapted to other activation functions that can be represented or approximated as the solution to a quadratic ODE. For example, the activation functions
\begin{equation*}
	x \mapsto \tanh(x),\quad x \mapsto \arctan(x),\quad x \mapsto \ln(1+e^x),
\end{equation*}
among others, can all be generated as a time-1 flow of a suitable system of quadratic ODEs. Step 2 in the proof will be replaced by a new system of ODEs adapted for a particular activation function, while Steps 1 and 3 will remain the same.

\section{Conclusion}

We have investigated a class of neural ordinary differential equations (continuous-time dynamical systems with right-hand sides depending on a time-varying finite-dimensional vector of parameters) with at most quadratic nonlinearities in the dynamics. We have shown, using explicit constructions, that these quadratic neural ODEs can approximate Sobolev-smooth functions and functions that are expressible as feedforward deep neural nets, and quantified the complexity of these approximations.

Beyond alternative frameworks for quantifying the performance of machine learning algorithms, there are also other ways of thinking about parametrization which may lead to new theoretical results about deep learning models. Certain classes of ``idealized'' models --- based on, for instance, exotic activation functions or universal differential equations, both of which require some sort of abstract mathematical process such as continuous-time integration or other infinite-dimensional constructions --- can achieve universal approximation using a finite number of parameters, this number not depending on the desired approximation error tolerance. These models form (approximate) parametrizations of an infinite-dimensional class of objects, such as compactly-supported continuous functions, using a finite-dimensional parameter set. Instead of the dimensionality of the parameter set being used to quantify the expressivity of these model architectures, some measure of precision of the discretized representation of those parameters or other quantification of the complexity of the algorithms needed to approximately implement these models on a computer may provide a more appropriate quantitative measure of expressivity. Similarly, the richness of the parameter space for a model --- which could itself be an infinite-dimensional class of functions, for instance in controlled ODEs --- can perhaps be characterized in other ways than its dimension, providing new perspectives by which to understand and quantify expressivity.



\bibliography{QODE.bbl}

\end{document}